\documentclass[10pt,twocolumn,letterpaper]{article}

\def\paperID{2390} 
\def\printOption{final} 

\usepackage{cvpr}
\usepackage{times}

\usepackage{import}

\usepackage[utf8]{inputenc} 
\usepackage[T1]{fontenc}    

\usepackage{amsmath}
\usepackage{amsthm}
\usepackage{mathtools}
\usepackage{commath}
\usepackage{mathrsfs}
\usepackage{amssymb}
\usepackage{amsfonts}       
\usepackage{bbm}
\usepackage{bm}
\usepackage{fancyhdr}
\usepackage{cancel}
\usepackage{nicefrac}       
\usepackage{microtype}      

\usepackage{algorithm}
\usepackage[noend]{algpseudocode}
\usepackage{algorithmicx}
\usepackage{eqparbox}

\usepackage[pdftex]{graphicx}
\usepackage{epsfig}
\usepackage[utf8]{inputenc}
\usepackage{tabularx}
\usepackage{booktabs, multicol, multirow}	
\usepackage[table, xcdraw]{xcolor}

\usepackage{caption}
\usepackage{subcaption}

\usepackage{appendix}

\usepackage{cite}
\usepackage{url}            

\def\printFinal{final}

\ifx \printOption \printFinal
\usepackage[breaklinks=true,bookmarks=false]{hyperref}
\cvprfinalcopy
\else
\usepackage[pagebackref=true,breaklinks=true,letterpaper=true,colorlinks,bookmarks=false]{hyperref}
\ifcvprfinal\fi
\fi

\def\cvprPaperID{\paperID} 

\DeclareMathOperator*{\argmin}{arg\,min}

\newcommand{\R}{\ensuremath{\mathbb{R}}}

\newtheorem{proposition}{Proposition}[section]

\theoremstyle{definition}

\numberwithin{equation}{section}

\makeatletter
\def\BState{\State\hskip-\ALG@thistlm}
\makeatother

\title{Deep Metric Learning with Alternating Projections onto Feasible Sets}

\author{%
  O\u{g}ul Can$^{\star}$ \qquad Yeti Z. G\"{u}rb\"{u}z$^{\dagger}$ \qquad A. Ayd{\i}n Alatan$^{\ddagger}$ \\
  Center for Image Analysis (OGAM), Middle East Technical University, Turkey \\
  {\tt\small $^{\star}$ogul.can@metu.edu.tr, $^{\dagger}$ygurbuz@metu.edu.tr, $^{\ddagger}$alatan@metu.edu.tr }
}

\begin{document}

\maketitle

\begin{abstract}
During the training of networks for distance metric learning, minimizers of the typical loss functions can be considered as "feasible points" satisfying a set of constraints imposed by the training data. To this end, we reformulate distance metric learning problem as finding a feasible point of a constraint set where the embedding vectors of the training data satisfy desired intra-class and inter-class proximity. The feasible set induced by the constraint set is expressed as the intersection of the relaxed feasible sets which enforce the proximity constraints only for particular samples (a sample from each class) of the training data. Then, the feasible point problem is to be approximately solved by performing alternating projections onto those feasible sets. Such an approach introduces a regularization term and results in minimizing a typical loss function with a systematic batch set construction where these batches are constrained to contain the same sample from each class for a certain number of iterations. Moreover, these particular samples can be considered as the class representatives, allowing efficient utilization of hard class mining during batch construction. The proposed technique is applied with the well-accepted losses and evaluated on Stanford Online Products, CAR196 and CUB200-2011 datasets for image retrieval and clustering. Outperforming state-of-the-art, the proposed approach consistently improves the performance of the integrated loss functions with no additional computational cost and boosts the performance further by hard negative class mining.

\end{abstract}


\section{Inroduction}

Distance metric learning (DML) is the problem of finding a proper function that satisfies metric axioms and assesses the semantic dissimilarity of the data samples from its domain. This task is generally realized by learning proper representations for the data samples so that the semantically similar ones are embedded to the small vicinity in the representation space as the dissimilar samples are placed relatively apart in the Euclidean sense. The representations are learned through an optimization framework in which the objective function utilizes the loss terms to impose the desired intra-class and inter-class proximity constraints in the representation space \cite{hu2014discriminative, schroff2015facenet, wu2017sampling, song2017deep, sohn2016improved, hadsell2006dimensionality, oh2016deep, Wang_2019_CVPR_MS, Qian_2019_ICCV}. The optimization is performed with mini-batch gradient updates and the procedure is generally guided by providing deliberately selected exemplars \cite{schroff2015facenet,yuan2017hard,sohn2016improved, wu2017sampling,harwood2017smart,ge2018deep, Suh_2019_CVPR,Wang_2019_CVPR, Wang_2019_CVPR_MS}.

\begin{figure}[t]

\begin{minipage}[b]{1.0\linewidth}
  \centering
  \centerline{\includegraphics[width=\linewidth]{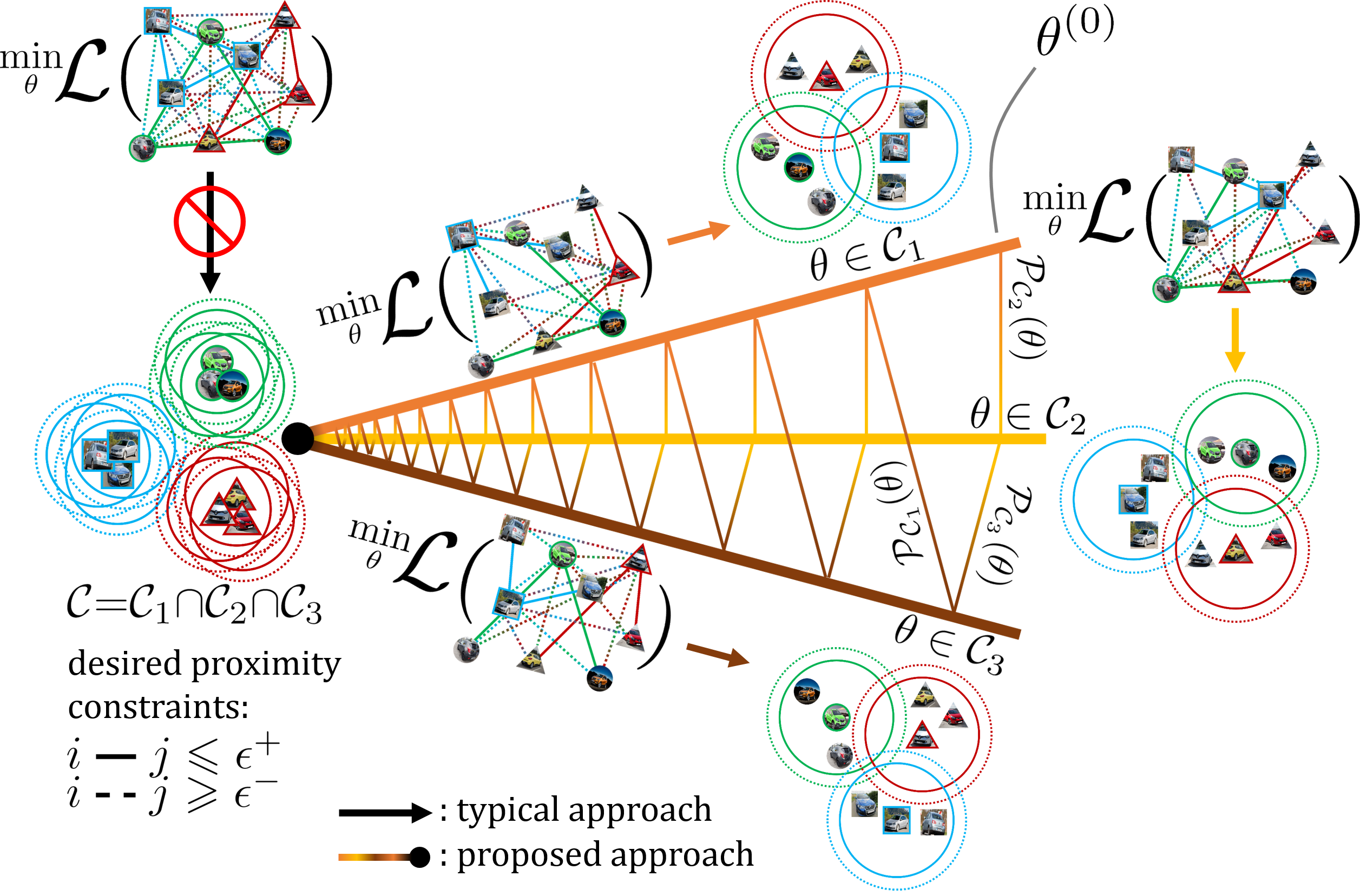}}

\end{minipage}

\caption{Proposed approach to DML problem. Instead of directly minimizing a loss function composed of the penalty terms enforcing the all proximity constraints (left), we alternatively minimize loss functions of proximity constraints only for particular samples (outlined samples). Each of the three lines (orange, yellow, brown) depicts a subset where the proximity constraints are satisfied for outlined samples as in Eq. \eqref{eq:supersets}. The solutions are related by projection to obtain a solution at the intersection. (Best viewed once magnified).}
	\label{fig:the_figure_1}
\end{figure}

Existing approaches focus on inventing loss functions to enforce all the proximity constraints at once. However, mini-batch gradient update nature of the optimization procedure implies alternatively considering only the subsets of the proximity constraints and eventually, the obtained representations are fail to satisfy the desired proximity constraints holistically due to possible traps in local minima.

To alleviate this problem, we revisit the proximity constraints in the representation space implied by the loss terms for proper DML. To develop a general framework for DML, we focus on finding a feasible point satisfying the proximity constraints. Such an alternative track is novel in the development of the DML frameworks and addresses the challenge of satisfying the proximity constraints for all data pairs to better match the distance in the representation space with semantic dissimilarity.

In contrast to existing methods, we approach the problem by posing it as a set intersection problem and propose to solve it by performing alternating projections onto the relaxed sets defined by the subsets of the proximity constraints. Our formulation results in relatively easier subproblems to be solved by minimizing the regularized version of the typical loss functions for DML with a systematic batch construction, where the batches are constrained to contain a particular sample from each class for a certain number of iterations. Not only that structure allows efficient utilization of hard negative class mining (HNCM) to guide the optimization without offline processing during batch construction but also the subproblems better fit the mini-batch gradient update procedure and are less prone to be trapped in local minima. Different from the existing methods, the subproblems are related by a regularization term owing to projection of the solutions to the subsets. Our approach to the DML problem is depicted in Fig. \ref{fig:the_figure_1}.

The implications and contributions as the results of our formulation are 1) a general framework that better exploits the proximity constraints to improve the performances of the current and possibly future DML loss functions, 2) idea of re-utilization of the particular class samples in the consecutive mini-batches during the optimization iterations, 3) a solid background through why such a biased batch construction should work, 4) relating the mini-batch updates in general with a regularization term and 5) an efficient class mining method for the batch sampling with $\mathcal{O}(L)$ complexity in contrast with similar methods \cite{harwood2017smart, ge2018deep, Suh_2019_CVPR} of $\mathcal{O}(N^2)$. 

\section{Notations and Definitions}
We consider dataset, $\lbrace (x_i, y_i)_i{\in}\mathcal{X}{\times}\mathcal{Y}\mid i{\in}\mathcal{N}\rbrace$, of two-tuples, where $x_i{\in}\mathcal{X}$ denotes a sample vector from the data space (\eg images), $y_i{\in}\mathcal{Y}{=}\lbrace 1,{\cdots},L\rbrace$ denotes the corresponding label of the sample among $L$ many classes and $\mathcal{N}{=}\lbrace 1,{\cdots},N\rbrace$ denotes the set of indexes to represent samples from the dataset of size $N$. Indicator of the two samples indexed by $i$ and $j$ belonging to the same class is denoted as $y_{i,j}{\in}\lbrace 0,1\rbrace$ where $y_{i,j}{=}1$ if $y_i{=}y_j$. We call $j$ positive/negative sample for $i$ if $y_{i,j}{=}1/0$.

The parametric distance between $x_i$ and $x_j$ is defined as:
\begin{equation}\label{eq:def:d}
d_{i,j}^f(\theta) \triangleq \Vert f(x_i;\theta)-f(x_j;\theta)\Vert_2
\end{equation}
which is the  Euclidean distance equipped with a $D$-dimensional vector valued parametric function, $f:\mathcal{X}{\xrightarrow{f}}\R^D$, with parameters $\theta$. The projection of $\theta$ onto a set $\mathcal{S}$ is:
\begin{equation}\label{eq:def:P}
\mathcal{P}_{\mathcal{S}}(\theta) \triangleq \argmin_{\vartheta{\in}\mathcal{S}} \tfrac{1}{2}\,\Vert \theta - \vartheta \Vert_2^2\,.
\end{equation}
	
For a set defined by an inequality $\mathcal{S} = \lbrace \theta\mid g(\theta) \leqslant 0 \rbrace$, we denote its indicator as $\iota_{\mathcal{S}}(\theta)$ which is:
\begin{equation}\label{eq:def:Ind}
\iota_{\mathcal{S}}(\theta) \triangleq \lim_{\lambda \to 0}\,[\tfrac{1}{\lambda}\,g(\theta) ]_+\,,
\end{equation}
where $[z]_+{=}\max\lbrace 0, z\rbrace$ and $g(\cdot)$ is an arbitrary function defining $\mathcal{S}$. We are to approximate $\iota_{\mathcal{S}}(\theta)$ for small $\lambda$ as:
\begin{equation}\label{eq:def:apprxInd}
\iota_{\mathcal{S}}(\theta)\approx \tfrac{1}{\lambda}\,[g(\theta) ]_+ \triangleq \hat{\iota}_{\mathcal{S}}(\theta).
\end{equation}

\section{Review of the Related Works}
\label{sec:review}
We restrict ourselves to the distance metric learning problem which is posed as learning the parameters, $\theta$, of an embedding function, $f(\cdot;\theta)$, so that the parametric distance, $d_{i,j}^f(\theta)$, between the data samples reflects their semantic dissimilarity. $f$ as a linear mapping is considered in earlier approaches \cite{xing2003distance, weinberger2006distance, perrot2015regressive} that later inspire most of state-of-the-art frameworks \cite{hu2014discriminative, schroff2015facenet, wu2017sampling, song2017deep, sohn2016improved, hadsell2006dimensionality, oh2016deep, Wang_2019_CVPR_MS, Qian_2019_ICCV} in which $f$ is a nonlinear mapping realized by deep neural networks. The general framework for learning the parameters is to minimize an overall loss function of the proximity constraints:
\begin{equation} \label{eq:general_framework}
\mathcal{L}^{(\cdot)}(\theta;\mathcal{T}) = \tfrac{1}{\vert\mathcal{T}\vert}\textstyle\sum\limits_{t\in\mathcal{T}}\ell_t^{(\cdot)}(\theta),
\end{equation}
where $\mathcal{T}$ is the set of index tuples (\eg pair $(i,j)$, triplet  $(i,j^+,j^-)$, etc.) and $\ell_t^{(\cdot)}(\theta)$ is a loss term penalizing the ranking violations among the samples indexed  by $t$. We omit dataset, $\mathcal{D}$, dependency of, $\mathcal{L}$, for clarity.

Learning a proper linear mapping for the parametric distance is initially formulated as a convex optimization problem in \cite{xing2003distance}. To prevent null mappings, a constraint that enforces mapping of the samples from different classes to be at least separated by some margin is added to the formulation. Moving that constraint to the objective via hinge loss \cite{hadsell2006dimensionality} results in the well-known contrastive loss:
\begin{equation}\label{eq:contrastive}
\ell^{cntrv}_{i,j}(\theta) = y_{i,j}\,d_{i,j}^f(\theta)^2 + (1-y_{i,j})\,[\varepsilon-d_{i,j}^f(\theta)]^2_+\, ,
\end{equation}
for sample pairs, $(i,j)$. Contrastive loss ignores intra-class variations of the classes. Triplet loss introduced in \cite{weinberger2006distance} and popularized in deep metric learning frameworks \cite{schroff2015facenet, harwood2017smart} alleviates this problem by constraining the distance to any positive sample to be at least some margin smaller than the distance to any negative sample for each sample:
\begin{equation}\label{eq:triplet}
\ell^{triplet}_{i,j^+,j^-}(\theta) = [d_{i,j^+}^f(\theta)^2 - d_{i,j^-}^f(\theta)^2 + \varepsilon]_+,
\end{equation}
where $j^+$ and $j^-$ are a positive and a negative sample for the sample $i$, respectively. Minimizing triplet loss entails deliberately selection of the triplets to have nonzero loss terms. Thus, either large batch size or mining for exemplars violating the triplet constraint is required \cite{schroff2015facenet}. Such an effort makes the computation of the triplet loss is less attractive than of the contrastive loss. Margin-based loss is introduced in \cite{wu2017sampling} to provide the flexibility in the distribution of the classes in the embedding space without using triplets as the exemplars. It expresses the margin constraint of the triplet loss as separate loss terms of the distances between positive and negative sample pairs by relaxing the constraint of the contrastive loss on the positive pairs:
\begin{equation}\label{eq:margin}
\begin{split}
\ell^{margin}_{i,j}(\theta)&=y_{i,j}\,[d_{i,j}^f(\theta) - \varepsilon + \delta]_+ \\
&+ (1-y_{i,j})\,[\varepsilon + \delta - d_{i,j}^f(\theta)]_+,
\end{split}
\end{equation}
where $\delta$ controls the separation margin and $\varepsilon$ is a trainable parameter for the boundary between positive and negative pairs. The contrastive, triplet and margin-based loss terms are the simplest forms of the pairwise distance ranking based loses. Proceeding approaches utilize smoothed versions of these losses by replacing hinge loss with log-sum-exp \cite{Wang_2019_CVPR, yi2014deep} or soft-max \cite{sohn2016improved, Yu_2019_ICCV} expression. In a different aspect, angular loss \cite{wang2017deep} that constraints the local geometry of the samples in the embedding space is proposed to better exploit the relation among triplets. In those losses, only 2 or 3 data samples contribute to the loss terms. Ranking among more samples are considered in quadruplet \cite{ law2013quadruplet, huang2016local, chen2017beyond}, histogram \cite{Cakir_2019_CVPR, ustinova2016learning} and soft-batch-mining \cite{Wang_2019_CVPR, oh2016deep, sohn2016improved, Yu_2019_ICCV} based losses. Soft-batch-mining exploits log-sum-exp expression as the approximation of max operator to select samples from the batch. Thus, ranking among multiple samples are considered.

The aforementioned approaches consider minimization of an overall loss function to enforce proximity constraints. However, the optimization is performed upon mini-batches due to the vast amount of constraints. Therefore, intra-class variations are prone to be missed, leading to poor generalization \cite{Jacob_2019_ICCV, lin2018deep, Roth_2019_ICCV}. Utilizing informative tuples with non-trivial settings in the gradient updates is considered to address the side effects of the mini-batch gradient updates to improve the representations \cite{schroff2015facenet, yuan2017hard, sohn2016improved, wu2017sampling, ge2018deep, harwood2017smart, Suh_2019_CVPR}. Yet, global mining \cite{ge2018deep, harwood2017smart, Suh_2019_CVPR} for such informative tuples brings additional computational burden and limits scalability to the large datasets, on the other hand, approximate mining \cite{schroff2015facenet, yuan2017hard, wu2017sampling, sohn2016improved, Wang_2019_CVPR, Wang_2019_CVPR_MS} is limited to cover enough settings to capture variations in the data. Generation of synthetic hard negative samples through adversarial \cite{duan2018deep, zhao2018adversarial} and variational \cite{Zheng_2019_CVPR, lin2018deep} models are addressed in recent approaches. 

To enhance the diversity of the semantic information embedded to the vector representations, ensemble techniques are also combined with deep metric learning framework \cite{opitz2017bier, kim2018attention, xuan2018deep, Sanakoyeu_2019_CVPR}. The general idea is simply concatenating the vectors from multiple embedding functions whose parameters are learned by considering different local features of the samples. Hence, better embedding space can be obtained by integrating the vectors that are specialized to different aspects of the samples. Modeling of the class variations through clustering track are also addressed in \cite{Qian_2019_ICCV, rippel2015metric, duan2017deep}. Rather than modeling, eliminating intra-class variances is considered in \cite{lin2018deep, Roth_2019_ICCV, Do_2019_CVPR, Jacob_2019_ICCV}. Though the general framework is to disentangle the intra-class variance upon global representation, differently in \cite{Jacob_2019_ICCV}, variations in the local features are addressed and higher-order moments are considered to regularize the local features so that their aggregation is of less variance. 

Dealing with the relaxed proximity constraints for better representation learning is not only the motivation of us but also the motivation of the recent approaches \cite{movshovitz2017no, Qian_2019_ICCV, Chen_2019_ICCV, Do_2019_CVPR}. The underlying idea is to better solve the relatively simpler problems. Existing approaches either are limited to the solution of the relaxed problem \cite{Do_2019_CVPR, Qian_2019_ICCV, movshovitz2017no} or fail to effectively exploit the relations among the relaxed problems to improve the solutions of the global problem \cite{Sanakoyeu_2019_CVPR,Chen_2019_ICCV}. Contrarily to existing approaches, our motivation is to improve generalization of the learned representations by satisfying proximity constraints more holistically through effectively combining the solutions of the simpler relaxed problems. To this end, we propose a feasible point formulation to inherently relate the relaxed problems and introduce a regularizer to be exploited not only in the proposed but also in a typical DML framework.

\section{Proposed Approach}
We revisit the early DML formulations \cite{xing2003distance,weinberger2006distance} which involve pairwise proximity constraints in the embedding space for the negative pairs. Yet, differently, we omit the positive pair distances to be minimized from the objective function and introduce it as a constraint for the constraint set of  $\theta$. Given a dataset of $N$ samples, $\lbrace (x_i, y_i)_i\rbrace_{i{\in}\mathcal{N}}$, we consider the constraint set $\mathcal{C}=\mathcal{C}^{+}\cap\mathcal{C^{-}}$ for $\theta$:
\begin{subequations}\label{eq:constraint_sets}
\begin{align}
\mathcal{C}^+ &=\lbrace\theta\mid d_{i,j}^f(\theta)\leqslant\varepsilon^+,\,\forall (i,j){\in}\mathcal{N}^2,\,y_{i,j}{=}1\rbrace \,, \\
\mathcal{C}^- &=\lbrace\theta\mid d_{i,j}^f(\theta)\geqslant\varepsilon^-,\,\forall (i,j){\in}\mathcal{N}^2,\,y_{i,j}{=}0\rbrace \,.
\end{align}
\end{subequations}
Once $\epsilon^+{<}\epsilon^-$, the constraints $\mathcal{C}^+$ and $\mathcal{C}^-$ together enforce $f(\cdot;\theta)$ to map the samples from the same class to some neighborhood such that no sample from any other class can be mapped to that neighborhood.
\begin{proposition} \label{proposition}
Any $\theta \in \mathcal{C}$ for some $\varepsilon^+$ and $\varepsilon^-$ is a global minimizer of the loss function, $\mathcal{L}^{(\cdot)}(\theta; \mathcal{T})$, defined in Eqn. \eqref{eq:general_framework} for the loss terms $\ell^{cntrv}_{i,j}(\theta)$, $\ell^{triplet}_{i,j^+,j^-}(\theta)$ and $\ell^{margin}_{i,j}(\theta)$ defined in Eqns. \eqref{eq:contrastive}-\eqref{eq:margin}, respectively. 
\end{proposition}
\begin{proof}[Proof]
The loss terms defined in Eqns. \eqref{eq:contrastive}-\eqref{eq:margin} satisfy $\ell^{(\cdot)}_t(\theta){\geqslant}0,\,\forall\theta$ which implies $\min\mathcal{L}^{(\cdot)}(\theta; \mathcal{T}){=}0$. Thus, it is enough to show that the proper selection of $\varepsilon^+$ and $\varepsilon^-$ yields $\ell^{(\cdot)}_t(\theta){=}0,\,{\forall}t{\in}\mathcal{T},\,{\forall}\theta{\in}\mathcal{C}$. If $\varepsilon^-{=}\varepsilon$, any $\theta {\in} \mathcal{C}$ results in $\ell^{cntrv}_{i,j}(\theta){\to}0$ as $\varepsilon^+{\to}0,\,{\forall}(i,j){\in}\mathcal{T}$. Similarly, choosing $\varepsilon^+{=}\varepsilon{-}\delta$ and $\varepsilon^-{=}\varepsilon{+}\delta$ makes $\ell^{margin}_{i,j}(\theta){=}0,\,{\forall}(i,j){\in}\mathcal{T},\,{\forall}\theta{\in}\mathcal{C}$. Finally, for any $\varepsilon^+$, choosing $\varepsilon^-{=}((\varepsilon^+)^2{+}\varepsilon)^{\nicefrac{1}{2}}$ results in $\ell^{triplet}_{i,j^+,j^-}(\theta){=}0,\,{\forall}(i,j^+,j^-){\in}\mathcal{T},\,{\forall}\theta{\in}\mathcal{C}$. 
\end{proof}
Proposition \ref{proposition} can be extended to the other pairwise distance based losses and suggests that finding a feasible point of $\mathcal{C}$ is equivalent to solving the minimization of those loss functions. This is actually the restatement of the motivation of the existing approaches \cite{hadsell2006dimensionality,schroff2015facenet, wu2017sampling, wang2017deep, ustinova2016learning, song2017deep, law2017deep, sohn2016improved, movshovitz2017no} in which the loss functions are developed to impose those constraints in the first place. Therefore, these methods can be considered as implicitly finding a feasible point of the constraint set via developing a loss function to be minimized. We address the problem differently by directly focusing on finding a feasible point and develop the loss function accordingly. To formulate our approach, we consider the relaxed set $\mathcal{C}_k=\mathcal{C}^{+}_k\cap\mathcal{C}^{-}_k$:
\begin{equation}\label{eq:supersets}
\begin{split}
\mathcal{C}^+_k &{=}\lbrace\theta\mid d_{k_l,j}^f(\theta)\leqslant\varepsilon^+,\,\forall (l,j),\,l{\in}\mathcal{Y},\,j{\in}\mathcal{N},\,y_{k_l,j}{=}1\rbrace \,, \\
\mathcal{C}^-_k &{=}\lbrace\theta\mid d_{k_l,j}^f(\theta)\geqslant\varepsilon^-,\,\forall (l,j),\,l{\in}\mathcal{Y},\,j{\in}\mathcal{N},\,y_{k_l,j}{=}0\rbrace \,, 
\end{split}
\end{equation}
where $k_l{\in}\lbrace i{\in}\mathcal{N}\mid y_i{=}l \rbrace$ denotes the index of a sample from class $l$. Note that the set $\lbrace k_l\rbrace_{l{=}1}^L$ contains a sample index from each class. For all samples, $\mathcal{C}_k$ enforces proximity constraints only relative to the particular class samples indexed by $\lbrace k_l\rbrace_{l{=}1}^L$. For each $k$, we consider distinct samples from each class such that $\lbrace k_l\rbrace_{l{=}1}^L{\cap}\lbrace k_l^{\prime}\rbrace_{l{=}1}^L{=}\emptyset$ for $k^{\prime}{\neq}k$. Then, $\mathcal{C}$ can be expressed as $\mathcal{C} = \displaystyle\cap_{k{=}1}^K \mathcal{C}_k$, where $K$ is the total number of sets, which can be considered as the maximum number of samples for a class. Then, the feasible point problem can be reformulated as finding a point in the intersection of the sets. If the sets, $\lbrace\mathcal{C}_k\rbrace_k$, were closed and convex, the problem would be solvable by alternating projection methods \cite{bregman1967relaxation, bauschke2000dykstras}. However, it is not uncommon to perform alternating projection methods to non-convex set intersection problems \cite{pang2015nonconvex, solomon2015convolutional}. Hence, we propose to solve the problem approximately by performing alternating projections onto the feasible sets $\lbrace\mathcal{C}_k\rbrace_k$. Hence, the problem becomes:
\begin{equation} \label{eq:soln}
\theta^\ast = \lim_{k\to \infty} \theta^{(k)},\text{ where } \theta^{(k)}= \mathcal{P}_{\mathcal{C}_k}(\theta^{(k{-}1)})\,,
\end{equation}
with $\mathcal{C}_{k{+}K} {\triangleq} \mathcal{C}_{k}$ and $\theta^{(0)}$ is arbitrary. A problem instance corresponding to a projection  becomes:
\begin{equation}
\theta^{(k)} = \mathcal{P}_{\mathcal{C}_k}(\theta^{(k{-}1)}) = \argmin_{\theta{\in}\mathcal{C}_k} \tfrac{1}{2}\,\Vert \theta^{(k{-}1)} {-} \theta \Vert_2^2\,,
\end{equation}
which can be written as an unconstraint problem in terms of the set indicator functions in \eqref{eq:def:Ind} as:
\begin{equation}
\begin{split}
&\theta^{(k)} {=} \argmin_{\theta} \tfrac{1}{2}\,\Vert \theta^{(k{-}1)} {-} \theta \Vert_2^2\\
&+  \textstyle\sum\limits_{(l,j){\mid}y_{k_l,j}{=}1}\iota_{\mathcal{S}^+_{k_l,j}}(\theta)
+\textstyle\sum\limits_{(l,j){\mid}y_{k_l,j}{=}0}\iota_{\mathcal{S}^-_{k_l,j}}(\theta)\,,
\end{split}
\end{equation}
where $\mathcal{S}^+_{k_l,j}{=}\lbrace \theta{\mid} d_{k_l,j}^f(\theta){\leqslant}\varepsilon^+\rbrace$ , $\mathcal{S}^-_{k_l,j}{=} \lbrace \theta{\mid} d_{k_l,j}^f(\theta){\geqslant}\varepsilon^-\rbrace$ and $(l,j){\in}\mathcal{Y}{\times}\mathcal{N}$. If the set indicator functions are to be approximated for small $\lambda$ as $\hat{\iota}_{\mathcal{S}_\ast^\mp}(\theta){=}\tfrac{1}{\lambda}\,[\mp(d^f_\ast(\theta){-}\varepsilon^\mp) ]_+$, as in Eqn. \eqref{eq:def:apprxInd}, the problem becomes after scaling with $\lambda$:
\begin{equation}\label{eq:loss_formulation}
\begin{split}
\theta^{(k)} &{=} \argmin_{\theta} \textstyle\sum\limits_{k_l,j}\big( y_{k_l,j}\,[d_{k_l,j}^f(\theta){-}\varepsilon^+]_+ \\
&{+} (1{-}y_{k_l,j})\, [\varepsilon^- {-} d_{k_l,j}^f(\theta)]_+\big) + \tfrac{\lambda}{2}\,\Vert \theta^{(k{-}1)} {-} \theta \Vert_2^2 
\end{split}
\end{equation}
where $(l,j){\in}\mathcal{Y}{\times}\mathcal{N}$. The resultant minimization problem for a projection step is very similar to a typical DML formulation in Eqn. \eqref{eq:general_framework} with a margin-based loss term \cite{wu2017sampling} defined in Eqn. \eqref{eq:margin}. The two significant differences are the utilization of the particular class samples, $\lbrace k_l\rbrace_{l{\in}\mathcal{Y}}$, for the pairwise distance losses and the regularization term relating the alternating subproblems.

\subsection{Solving for Parameters}
To obtain the solution, $\theta^{\ast}$, defined in Eqn. \eqref{eq:soln}, one should cycle through the sets, $\lbrace\mathcal{C}_k\rbrace_k$, and perform projections until the convergence. The nature of the problem is non-convex. Therefore, exploiting diverse combinations of sets might improve the solution. We propose to perform projections by randomly selecting the class samples for the sets. This approach results in different feasible sets $\lbrace\mathcal{C}_k\rbrace_k$ for each cycle so that the procedure does not stick to the specific sets. Performing a projection involves a minimization problem. In this perspective, either convergence can be monitored to pass the next projection or the projection operator can be approximated by $M$ iterations of training. The latter approach gives the flexibility to control the fitting of the parameters to the subproblems. We therefore propose to use $M$-step approximations of the projection operators as $\theta^{(k)}{\approx} \mathcal{P}^{(M)}_{\mathcal{C}_k}(\theta^{(k{-}1)})$ in our framework.

Proposed learning procedure for $\theta$ necessitates utilization of the particular class samples for the loss computation. To provide scalability, those particular class samples can also be sampled during batch construction. In this manner, another implication of the proposed framework becomes imposing constraint on the batch construction for the minimization. If we disregard the resultant loss formulation in Eq. \eqref{eq:loss_formulation} and consider only batch construction method of our framework, we can formulate minimization of any pairwise distance ranking based loss function within the proposed batch construction method as:
\begin{equation}
\theta^{(k)} = \argmin_{\theta} \mathcal{L}^{(\cdot)}(\theta; \mathcal{T}_k) + \tfrac{\lambda}{2}\,\Vert \theta^{(k{-}1)} - \theta \Vert_2^2\,,
\end{equation}
where $(\cdot)$ can be any proper loss and $\mathcal{T}_k$ is the tuple set of the sample indices defining $\mathcal{C}_k$. The proposed DML framework in its most general form is summarized in Algorithm \ref{algo:profs}. Alternating projections only introduces a constrained batch construction step to the standard optimization procedure.
\algblockdefx{MRepeat}{EndRepeat}{\textbf{repeat} }{}
\algnotext{EndRepeat}

\renewcommand{\algorithmiccomment}[1]{\hfill\eqparbox{COMMENT}{// #1}}

\begin{algorithm}
\caption{PROFS DML}\label{algo:profs}
\begin{algorithmic}
\State randomly initialize parameters $\theta^{(0)}$, $k{=}0$
\Repeat
\State sample $\mathcal{R}{=}\lbrace k_l\mid y_{k_l}{=}l \rbrace_{l{=}1}^L{\sim}\mathcal{N}$ class representatives
\\ \Comment{\ie an example for each class}
\MRepeat {$M$ times} \Comment{$\theta^{(k{+}1)}{\approx}\mathcal{P}^{(M)}_{\mathcal{C}_k}(\theta^{(k)})$}
\State sample $\mathcal{R}^{\prime}{=}\lbrace i_r\rbrace_r{\sim}\mathcal{R}$ a subset of the class reps.
\\ \Comment{possibly via hard class mining}
\State sample $\mathcal{E}{=}\lbrace j_n\rbrace_n{\sim}\mathcal{N}$ a batch of examples
\State construct exemplar batch $\mathcal{B}$ from $\mathcal{R}^{\prime}$ and $\mathcal{E}$
\\ \Comment{\eg pairs $(i_r,j_n)$}
\State $\theta{\gets} \theta {-} \alpha\nabla_{\theta}( \tfrac{\lambda}{2}\,\Vert \theta^{(k)} {-} \theta \Vert_2^2 {+} \mathcal{L}^{(\cdot)}(\theta;\mathcal{B}))$
\EndRepeat
\State $k{\gets} k {+} 1$, $\theta^{(k)}{\gets}\theta$
\Until {convergence}

\end{algorithmic}
\end{algorithm}
%

\subsection{Implications}
\textbf{Robustness.} The parameters are obtained through solving then relating the subproblems rather than solving an overall problem at once. The subproblems are relatively easier problems owing to the relaxed proximity constraints to be considered. Thus, better solutions to the subproblem specific proximity constraints can be obtained. Nevertheless the solutions are expected to be localized, the regularization term entangles the solutions of the alternating subproblems for a better holistic solution. This nature makes our approach more robust to the mini-batch updates. 

\textbf{Class representatives.} The subproblems can be seen as learning representations so that the particular class samples become the class representatives, akin to learning class representative vectors for a softmax classifier. Therefore, choosing the relaxed feasible sets as in Eqn. \eqref{eq:supersets} implicitly integrates classification framework to the DML problem in our formulation. In this point of view, our formulation is aligned with the studies \cite{perrot2015regressive, sun2014deep,zhang2016embedding, Qian_2019_ICCV, movshovitz2017no} supporting the superiority of the softmax incorporation into DML. 

\textbf{Hard negative class mining.} The nature of the proposed DML framework allows efficient HNCM batch construction. As $\theta$ is projected onto the feasible sets, the set-specific class samples can be considered as the class representatives. Thus, using those representatives, approximate global mining can be efficiently performed. To this end, we store the embedding representations of the class representatives and update those representations as we sample the corresponding representatives for the batch construction. In this way, we can perform online HNCM with $\mathcal{O}(L)$ complexity.

\textbf{Relation to linear metric learning with convex optimization.} Convex optimization formulation of linear metric learning in \cite{xing2003distance} involves similar alternating projections onto feasible sets to perform projected gradient ascent. That approach only considers the entire constraint set induced by the positive pairs, whereas we consider both negative and positive pair distances jointly in the relaxed feasible sets.

\textbf{Relation to tuplet and proxy-based losses.} Minimization of tuplet losses \cite{sohn2016improved, Yu_2019_ICCV} can be considered as performing our method by using $M{=}1$ step approximation of the projection operator. Similarly, the minimization of proxy-based losses \cite{movshovitz2017no, perrot2015regressive, Qian_2019_ICCV, Do_2019_CVPR} can be considered as a single projection operation. Therefore, tuplet and proxy-based loss formulations are to be the two extreme cases of our framework.

\textbf{Regularization.} Our formulation suggests a regularization term to entangle the solutions of the different subproblems. The gradient update of a typical DML framework exploiting deliberately constructed mini-batches during the loss minimization can be considered as performing projections with $M{=1}$ step approximation. The proposed regularization term can be introduced to the loss to improve generalization by entangling the updates of the different batches.

\section{Experimental Work}
We examine the effectiveness of the proposed DML framework through evaluation on the three widely-used benchmark datasets for the image retrieval and clustering tasks. We perform ablation study on the effect of $M$-step approximation of the projection operation and the regularization term. Throughout the section, we use PROFS to refer our framework.

\subsection{Benchmark Datasets and Evaluation Metrics}
\label{sec:dataset}
We obtain results by utilizing three public benchmark datasets. The conventional protocol of splitting training and testing sets for a zero-shot  setting \cite{oh2016deep} is followed for all datasets. Hence, no image is in the intersection of the training and the test sets. Stanford Online Products (SOP) \cite{oh2016deep} dataset has 22,634 classes with 120,053 product images. The first 11,318 classes (59,551 images) are split for training and the other 11,316 (60,502 images) classes are used for testing. Cars196 \cite{krause2014submodular} dataset contains 196 classes of cars with 16,185 images. The first 98 classes (8,054 images) are used for training and remaining 98 classes (8,131 images) are reserved for testing. CUB-200-2011 \cite{wah2011caltech} dataset consists of 200 species of birds with 11,788 images. The first 100 species (5,864 images) are split for training, the rest of 100 species (5,924 images) are used for testing.

We follow the standard metric learning experimental protocol defined in \cite{oh2016deep} to evaluate the performance of the deep metric learning approaches for the retrieval and clustering tasks. We utilize normalized mutual information (NMI) and F$_1$ score to measure the quality of the clustering task which is performed by conventional $k$-means clustering algorithm. In order to evaluate clustering performance, normalized mutual information (NMI) and F$_1$ score are utilized. NMI computes the label agreement between predicted and groundtruth clustering assignments neglecting the permutations while F$_1$ measures harmonic mean of the precision and recall. Furthermore, we exploit binary Recall@K metric to evaluate the performance of the retrieval task. Recall@K for a test query is 1 if at least one sample from the same class of the query is in the  $K$ nearest neighborhood of the query. The average of the Recall@K for the test queries gives the Recall@K performance on the dataset. We refer \cite{oh2016deep} for the detailed information related to these evaluation metrics.

\subsection{Training Setup}
\label{sec:setup}
We use Tensorflow \cite{abadi2016tensorflow} and PyTorch \cite{paszke2017automatic} deep learning libraries throughout the experiments. Tensorflow is used for trainings on GoogLeNet V1 (Inception V1) \cite{szegedy2015going} and PyTorch is used for trainings on GoogLeNet V2 (with batch normalization) \cite{normalization2015accelerating} and ResNet-50 \cite{he2016identity}. After the images are normalized and scaled to $256{\times}256$, we perform $224{\times}224$ random crop and data augmentation by horizontal mirroring as  pre-processing. For the embedding function, $f(\cdot;\theta)$, we utilize architectures until the output of the global average pooling layer with the parameters pretrained on ImageNet ILSVRC dataset \cite{russakovsky2015imagenet}. After the pooling layer, we add a linear transformation layer (fully connected layer) to obtain the representation vectors of size 512. We fix the embedding size of the samples at 512 throughout experiments, since it is shown in \cite{oh2016deep} that the embedding size does not have a key role on comparing performances of the deep metric learning loss functions. The parameters of the linear transformation layer are randomly initialized and are learned by using 10 times larger learning rate than the pretrained parameters for the sake of fast convergence. For the hyper-parameters, our framework introduces 2 additional parameters: $\lambda$ for regularization and $M$ to approximate projection operation. We set the regularization term as the result of the projection based formulation, $\lambda$, to a small reasonable value, $10^{-3}$. The number of projections steps, $M$, is determined according to the findings of the ablation study which is presented in subsection \ref{sec:ablation}. For the other hyper-parameters coming from adaptation of the baseline framework (\eg margins, number of positive samples etc.), we follow the settings in the corresponding baseline work. For the optimization procedure, we select the base learning rate as $10^{-4}$ for SOP dataset whereas we utilize $10^{-5}$ learning rate to train CUB-200-2011 and Cars196, since they tend to meet over-fitting problem due to the limited dataset size. We exploit Adam \cite{kingma2014adam} optimizer for mini-batch gradient descent with a mini-batch size is 128 and default moment parameters, $\beta_1{=}.9$ and $\beta_2{=}.99$. Finally, since the convergence rate of each method is different, we train all the approaches for 100 epochs and post the performance at their best epoch as in \cite{wu2017sampling} instead of following the conventional procedure \cite{oh2016deep} reporting performance of DML approaches after a certain number of training iterations.

\subsection{Baseline Methods and PROFS Framework Adaptation}
\label{sec:baseline}
We apply proposed PROFS framework with and without HNCM on the contrastive \cite{hadsell2006dimensionality}, triplet \cite{schroff2015facenet}, lifted structured \cite{oh2016deep}, $N$-pair \cite{sohn2016improved}, angular \cite{wang2017deep}, margin-based (Margin) \cite{wu2017sampling}, multi-similarity (MS) \cite{Wang_2019_CVPR_MS} and SoftTriple \cite{Qian_2019_ICCV} loss functions in order to directly compare with the state-of-the art methods. The comparison with the contrastive, triplet and Margin losses is important to examine the effectiveness of our original formulation, while the comparison with the other losses is to show that the formulation can be extended to the other loss functions by exploiting the proposed batch construction and regularization. Furthermore, we compare proxy-based \cite{movshovitz2017no, Qian_2019_ICCV} loss functions with PROFS owing to its relation to our formulation.

To evaluate the approaches in the same basis, we retrain all the aforementioned loss functions excluding SoftTriple by exploiting the same GoogLeNet V1 architecture with the default hyper-parameters used in the original works except for the mini-batch and the embedding sizes as explained in the subsection \ref{sec:setup} for a fair comparison. Moreover, we integrate our framework to Margin and SoftTriple losses in their own architectures (ResNet-50 and GoogLeNet V2, respectively) in order to compare our framework with state-of-the-art.

To adapt the lifted structured and MS loss to PROFS framework, the loss is slightly modified by ignoring the pairwise terms between the non-representative samples. For the adaptation of the sampling strategies, we trained Margin loss with the distance weighted sampling in its own architecture. However, we were unable to acquire good results in GoogLeNet V1, since the distance weighted sampling is very sensitive to its parameters and we could not determine well-performing hyper-parameters. On the other hand, due to its similarity with the contrastive loss, we exploit the same hard mining strategy inspired from \cite{yuan2017hard} as in the contrastive loss for the mini-batch sampling method of the Margin loss. It should be noted that we sample one negative pair for each positive pair as in \cite{wu2017sampling} for the contrastive, triplet and Margin loss functions. Such an approach provides balance to the number of positive and negatives pairs. In the conventional hard mining, the number of hard negative pairs should match the number of positive pairs for the contrastive, triplet and Margin loss functions. Thus, in PROFS framework, each class representative should have the same number of its corresponding positive pairs and hard negative pairs to obtain an exemplar set consistent with hard mining without violating the batch construction constraint of PROFS. No adaptation is performed for the $N$-pair and angular loss, since their formulation is consistent with PROFS. Finally, for SoftTriple loss, the adaptation is not straightforward, since this framework is inherently relaxed formulation of DML problem. On the other hand, in that framework, the loss terms are determined by assigning samples to the trainable cluster centers. At different iterations, the assignment may differ. To this end, we consider each iteration as a subproblem and integrate our framework by exploiting the regularization term to entangle the updates of the iterations. For HNCM, we utilize cluster centers as representatives in SoftTriple.

\begin{table*}[]
	\centering
	\caption{Comparison with the existing methods for the clustering and the retrieval tasks on SOP \cite{oh2016deep}, CARS196 \cite{krause2014submodular} and CUB-200-2011 \cite{wah2011caltech} datasets. Red: the overall best. Blue: the overall second best. Bold: the loss term specific best. }
	\label{tab:all}
	\resizebox{\textwidth}{!}{%
		\begin{tabular}{lcccccc|cccccc|cccccc}
			\toprule 
			&\multicolumn{6}{c|}{Stanford Online Products}& \multicolumn{6}{c|}{CARS196} & \multicolumn{6}{c}{CUB-200-2011} \\ \cmidrule{2-7} \cmidrule{8-13} \cmidrule{14-19} 
			Method & NMI & F$_1$ & R@1 & R@10 & R@100 & R@1000 
			 & \multicolumn{1}{l}{NMI} & \multicolumn{1}{l}{F$_1$} & \multicolumn{1}{l}{R@1} & \multicolumn{1}{l}{R@2} & \multicolumn{1}{l}{R@4} &\multicolumn{1}{l|}{R@8} & \multicolumn{1}{l}{NMI} & \multicolumn{1}{l}{F$_1$} & \multicolumn{1}{l}{R@1} & \multicolumn{1}{l}{R@2} & \multicolumn{1}{l}{R@4} & \multicolumn{1}{l}{R@8} \\ \midrule
			Contrastive-Hard & {\color[HTML]{000000} 89.7} & {\color[HTML]{000000} 34.5} & {\color[HTML]{000000} 67.9} & {\color[HTML]{000000} 83.8} & {\color[HTML]{000000} 93.2} & {\color[HTML]{000000} 97.9} & 66.0 & 36.6 & 75.8 & 84.5 & 90.1 & 94.1 &63.4 & 31.8 & 56.7 & 68.4 & 78.8 & 86.3 \\
			C-PROFS & {\color[HTML]{000000} 90.3} & {\color[HTML]{000000} 36.4} & {\color[HTML]{000000} 70.1} & {\color[HTML]{000000} 85.4} & {\color[HTML]{000000} 94.1} & {\color[HTML]{000000} 98.2} & \textbf{67.8} & \color[HTML]{000000}{\textbf{39.0}} & 76.0 & 84.8 & 90.1 & \textbf{94.6} & 64.1 & 32.0 & 57.7 & \textbf{69.0} & 78.9 & 86.9 \\
			C-PROFS-HNCM & {\color[HTML]{000000} \textbf{91.0}} & {\color[HTML]{000000} \textbf{41.2}} & {\color[HTML]{000000} \textbf{74.5}} & {\color[HTML]{000000} \textbf{87.9}} & {\color[HTML]{3166FF} \textbf{94.9}} & {\color[HTML]{FE0000} \textbf{98.3}} & 66.9 & 38.1 & \textbf{77.0} & \textbf{85.3} & \textbf{90.9} & 94.4 & \textbf{64.6} & \textbf{33.1} & \textbf{57.9} & \textbf{69.0} & \textbf{79.4} & \textbf{87.0} \\ \midrule
			Triplet-Semi & {\color[HTML]{000000} 87.1} & {\color[HTML]{000000} 23.4} & {\color[HTML]{000000} 57.0} & {\color[HTML]{000000} 75.0} & {\color[HTML]{000000} 88.2} & {\color[HTML]{000000} 96.4} & 62.4 & 30.0 & 71.2 & 80.7 & 87.6 & 92.5 & 60.9 & 27.8 & 55.5 & 67.8 & 78.2 & 86.5 \\
			Triplet-Hard & {\color[HTML]{000000} 88.1} & {\color[HTML]{000000} 30.6} & {\color[HTML]{000000} 65.4} & {\color[HTML]{000000} 81.4} & {\color[HTML]{000000} 91.7} & {\color[HTML]{000000} 97.4} & 62.5 & 30.6 & 71.4 & 81.1 & 87.5 & 92.7 & 63.1 & 30.7 & 56.8 & 68.7 & 78.6 & 86.5 \\
			T-PROFS & {\color[HTML]{000000} 90.7} & {\color[HTML]{000000} 37.7} & {\color[HTML]{000000} 72.4} & {\color[HTML]{000000} 86.3} & {\color[HTML]{000000} 94.1} & {\color[HTML]{000000} \textbf{98.1}} & 64.7 & 33.4 & \textbf{74.5} & 82.8 & 89.1 & 93.5 & 63.5 & 31.3 & 57.9 & 68.9 & 78.7 & 86.7 \\
			T-PROFS-HNCM & {\color[HTML]{3166FF} \textbf{91.3}} & {\color[HTML]{000000} \textbf{42.3}} & {\color[HTML]{3166FF} \textbf{74.9}} & {\color[HTML]{000000} \textbf{87.5}} & {\color[HTML]{000000} \textbf{94.4}} & {\color[HTML]{000000} \textbf{98.1}} & \textbf{64.9} & \textbf{33.7} & \textbf{74.5} & \textbf{83.5} & \textbf{89.4} & \textbf{93.7} & \textbf{64.5} & \textbf{32.3} & {\color[HTML]{3166FF} \textbf{58.1}} & \textbf{69.0} & \textbf{78.9} & \textbf{86.8} \\ \midrule
			Lifted & {\color[HTML]{000000} 88.9} & {\color[HTML]{000000} 31.4} & {\color[HTML]{000000} 66.7} & {\color[HTML]{000000} 83.2} & {\color[HTML]{000000} 91.7} & {\color[HTML]{000000} 97.4} & 60.1 & 27.7 & 67.5 & 77.3 & 84.9 & 90.7 & 60.6 & 26.9 & 53.5 & 65.3 & 75.4 & 84.8 \\
			L-PROFS & {\color[HTML]{000000} 89.5} & {\color[HTML]{000000} 33.8} & {\color[HTML]{000000} 68.1} & {\color[HTML]{000000} 84.0} & {\color[HTML]{000000} 92.2} & {\color[HTML]{000000} 97.6} & 61.2 & 27.9 & 68.4 & 78.1 & 85.6 & 91.1 & 61.4 & 28.1 & 54.5 & \textbf{66.1} & 76.2 & 85.1 \\
			L-PROFS-HNCM & {\color[HTML]{000000} \textbf{89.9}} & {\color[HTML]{000000} \textbf{35.1}} & {\color[HTML]{000000} \textbf{69.3}} & {\color[HTML]{000000} \textbf{85.1}} & {\color[HTML]{000000} \textbf{93.6}} & {\color[HTML]{000000} \textbf{98.1}} & \textbf{61.5} & \textbf{30.0} & \textbf{70.7} & \textbf{79.6} & \textbf{86.2} & \textbf{91.5} & \textbf{62.0} & \textbf{29.5} & \textbf{54.6} & \textbf{66.1} & \textbf{76.7} & \textbf{85.5} \\ \midrule
			$N$-pair & {\color[HTML]{000000} 89.9} & {\color[HTML]{000000} 35.7} & {\color[HTML]{000000} 70.8} & {\color[HTML]{000000} 86.0} & {\color[HTML]{000000} 94.0} & {\color[HTML]{000000} 98.1} & 67.4 & 38.2 & 76.7 & 84.8 & 91.0 & 95.0 & 64.6 & 33.0 & 56.0 & 68.9 & 79.3 & 87.4 \\
			$N$-PROFS & {\color[HTML]{000000} 90.3} & {\color[HTML]{000000} 37.3} & {\color[HTML]{000000} 71.7} & {\color[HTML]{000000} \textbf{86.6}} & {\color[HTML]{000000} 94.0} & {\color[HTML]{3166FF} \textbf{98.2}} & {\color[HTML]{000000} 68.1} & 38.3 & {\color[HTML]{3166FF} \textbf{77.8}} & {\color[HTML]{000000} 85.9} & {\color[HTML]{000000} 91.6} & {\color[HTML]{3166FF} \textbf{95.2}} & {\color[HTML]{000000} 64.9} & {\color[HTML]{000000} 33.6} & 56.5 & \textbf{69.0} & 79.3 & {\color[HTML]{FE0000} \textbf{87.7}} \\
			$N$-PROFS-HNCM & {\color[HTML]{000000} \textbf{90.5}} & {\color[HTML]{000000} \textbf{37.5}} & {\color[HTML]{000000} \textbf{71.9}} & {\color[HTML]{000000} \textbf{86.6}} & {\color[HTML]{000000} \textbf{94.1}} & {\color[HTML]{3166FF} \textbf{98.2}} & {\color[HTML]{FE0000} \textbf{68.6}} & {\color[HTML]{FE0000} \textbf{39.9}} & 77.6 & {\color[HTML]{FE0000} \textbf{86.2}} & {\color[HTML]{3166FF} \textbf{91.7}} & {\color[HTML]{3166FF} \textbf{95.2}} & {\color[HTML]{FE0000} \textbf{65.2}} & {\color[HTML]{000000} \textbf{33.7}} & \textbf{56.7} & 68.8 & {\color[HTML]{FE0000} \textbf{79.7}} & 87.5 \\ \midrule
			Angular & {\color[HTML]{000000} 90.0} & {\color[HTML]{000000} 36.1} & {\color[HTML]{000000} 72.5} & {\color[HTML]{000000} 86.6} & {\color[HTML]{000000} 93.6} & {\color[HTML]{000000} 97.6} & 66.0 & 35.9 & 77.4 & 85.3 & 91.0 & 94.7 & 61.9 & 29.4 & 54.5 & 66.4 & 76.8 & 84.9 \\
			A-PROFS & {\color[HTML]{000000} 90.1} & {\color[HTML]{000000} 37.2} & {\color[HTML]{000000} 73.0} & {\color[HTML]{000000} 86.8} & {\color[HTML]{000000} 93.7} & {\color[HTML]{000000} \textbf{97.7}} & \textbf{66.3} & 36.8 & 77.5 & \textbf{85.6} & 91.2 & 94.7 & 63.0 & 31.7 & 54.6 & 66.7 & 76.9 & 85.8 \\
			A-PROFS-HNCM & {\color[HTML]{000000} \textbf{90.4}} & {\color[HTML]{000000} \textbf{38.4}} & {\color[HTML]{000000} \textbf{73.7}} & {\color[HTML]{000000} \textbf{86.9}} & {\color[HTML]{000000} \textbf{93.8}} & {\color[HTML]{000000} \textbf{97.7}} & \textbf{66.3} & \textbf{37.9} & {\color[HTML]{FE0000} \textbf{77.9}} & \textbf{85.6} & \textbf{91.4} & \textbf{94.8} & \textbf{64.6} & \textbf{33.4} & \textbf{55.8} & \textbf{68.1} & \textbf{78.8} & \textbf{87.3} \\ \midrule
			Margin-Hard & {\color[HTML]{000000} 90.6} & {\color[HTML]{000000} 38.1} & {\color[HTML]{000000} 73.9} & {\color[HTML]{000000} 87.7} & {\color[HTML]{000000} 94.8} & {\color[HTML]{3166FF} 98.2} & 64.2 & 34.6 & 75.1 & 83.7 & 89.5 & 93.8 & 64.0 & 30.9 & 55.3 & 67.2 & 77.9 & 87.5 \\
			M-PROFS & {\color[HTML]{3166FF} 91.3} & {\color[HTML]{000000} 42.7} & {\color[HTML]{000000} 74.5} & {\color[HTML]{3166FF} 88.0} & {\color[HTML]{FE0000} \textbf{95.0}} & {\color[HTML]{3166FF} 98.2} & 64.6 & 35.1 & 76.0 & 84.3 & 89.5 & 93.8 & 64.3 & \textbf{32.7} & 57.7 & {\color[HTML]{3166FF} 69.5} & {\color[HTML]{000000} 79.5} & 87.5 \\
			M-PROFS-HNCM & {\color[HTML]{FE0000} \textbf{91.4}} & {\color[HTML]{FE0000} \textbf{43.2}} & {\color[HTML]{FE0000} \textbf{76.3}} & {\color[HTML]{FE0000} \textbf{88.8}} & {\color[HTML]{FE0000} \textbf{95.0}} & {\color[HTML]{FE0000} \textbf{98.3}} & \textbf{66.8} & \textbf{37.3} & \textbf{77.0} & \textbf{85.1} & \textbf{90.8} & \textbf{94.6} & \textbf{64.8} & 32.4 & {\color[HTML]{FE0000} \textbf{58.5}} & {\color[HTML]{FE0000} \textbf{69.6}} & {\color[HTML]{FE0000} \textbf{79.7}} & {\color[HTML]{3166FF} \textbf{87.6}} \\ 
			\midrule
			MS & {\color[HTML]{000000} 90.5} & {\color[HTML]{000000} 38.2} & {\color[HTML]{000000} 71.3} & {\color[HTML]{000000} 86.4} & {\color[HTML]{000000} 94.2} & {\color[HTML]{000000} 98.1} 
			& 65.3 & 35.5 & 76.2 & 84.2 & 89.9 & 94.0 & 
			64.7 & 33.8 & 56.4 & 69.1 & 79.3 & 87.5 \\
			MS-PROFS & {\color[HTML]{000000} 90.8} & {\color[HTML]{000000} 39.5} & {\color[HTML]{000000} 72.6} & {\color[HTML]{000000} 87.2} & {\color[HTML]{000000} 94.3} & {\color[HTML]{3166FF} \textbf{98.2}} 
			& 66.6 & 37.1 & 76.6 & 84.9 & 90.6 & 94.4 &
			\color[HTML]{3166FF} 65.0 & \color[HTML]{3166FF}34.3 & 57.5 & {\color[HTML]{000000} \textbf{69.4}} & {\color[HTML]{3166FF} \textbf{79.6}} & 87.5 \\
			MS-PROFS-HNCM & {\color[HTML]{000000} \textbf{91.0}} & {\color[HTML]{3166FF} \textbf{42.9}} & {\color[HTML]{000000} \textbf{74.6}} & {\color[HTML]{000000} \textbf{87.7}} & {\color[HTML]{000000} \textbf{94.6}} & {\color[HTML]{3166FF} \textbf{98.2}} &
			 \color[HTML]{3166FF}\textbf{68.4} & \color[HTML]{3166FF}\textbf{39.2} & \color[HTML]{3166FF}\textbf{77.8} & \color[HTML]{3166FF}\textbf{86.0} & \color[HTML]{FE0000}\textbf{91.8} & \color[HTML]{FE0000}\textbf{95.3} &
			  \color[HTML]{FE0000}\textbf{65.2} & \color[HTML]{FE0000}\textbf{34.4} & {\color[HTML]{3166FF} \textbf{58.1}} & {\color[HTML]{000000} 69.3} & {\color[HTML]{3166FF} \textbf{79.6}} & {\color[HTML]{FE0000} \textbf{87.7}} \\ \bottomrule
		\end{tabular}%
	}
\end{table*}

\begin{table*}[]
	\centering
	\caption{Comparison with state-of-the-art methods on SOP \cite{oh2016deep}, CARS196 \cite{krause2014submodular} and CUB-200-2011 \cite{wah2011caltech} datasets. Red: the overall best. Blue: the overall second best. Bold: the loss term specific best. }
	\label{tab:all_result_2}
	\resizebox{\textwidth}{!}{
	\begin{tabular}{lcccc|cccc|cccc} 
		\toprule
		& \multicolumn{4}{c|}{Stanford Online Products}                 & \multicolumn{4}{c|}{CARS196}                                                                           & \multicolumn{4}{c}{CUB-200-2011}                                                                       \\ 
		\cmidrule{2-13}
		Method                & R@1           & R@10          & R@100         & R@1000        & \multicolumn{1}{l}{R@1} & \multicolumn{1}{l}{R@2} & \multicolumn{1}{l}{R@4} & \multicolumn{1}{l|}{R@8} & \multicolumn{1}{l}{R@1} & \multicolumn{1}{l}{R@2} & \multicolumn{1}{l}{R@4} & \multicolumn{1}{l}{R@8}  \\ 
		\midrule
		Proxy-NCA  \cite{movshovitz2017no}           & 73.7          & -             & -             & -             & 73.2                    & 82.4                    & 86.4                    & 88.7                     & 49.2                    & 61.9                    & 67.9                    & 72.4                     \\
		Clustering \cite{song2017deep}           & 67.0          & 83.7          & 93.2          & -             & 58.1                    & 70.6                    & 80.3                    & 87.8                     & 48.2                    & 61.4                    & 71.8                    & 59.2                     \\
		HDC  \cite{yuan2017hard}                 & 69.5          & 84.4          & 92.8          & 97.7          & 73.7                    & 83.2                    & 89.5                    & 93.8                     & 53.6                    & 65.7                    & 77.0                    & 85.6                     \\
		HTL  \cite{ge2018deep}                 & 74.8          & 88.3          & 94.8          & 98.4          & 81.4                    & 88.0                    & 92.7                    & 95.7                     & 57.1                    & 68.8                    & 78.7                    & 86.5                     \\
		MS   \cite{Wang_2019_CVPR_MS}                 & 78.2          & 90.5          & 96.0          & 98.7          & 84.1                    & 90.4                    & 94.0                    & 96.5                     & 65.7                    & \color[HTML]{FE0000}77.0                    & \color[HTML]{FE0000}86.3                    & \color[HTML]{FE0000}91.2                     \\
		TML    \cite{Yu_2019_ICCV}               & 78.0          & 91.2          & \color[HTML]{3166FF}96.7          & 99.0          & \color[HTML]{FE0000}86.3                    & \color[HTML]{3166FF}92.3                    & \color[HTML]{FE0000}95.4                    & 97.3                     & 62.5                    & 73.9                    & 83.0                    & 89.4                     \\ 
		\hline
		Margin   \cite{wu2017sampling}            & 72.7          & 86.2          & 93.8          & 98.0          & 79.6                    & 86.5                    & 91.9                    & 95.1                     & 63.6                    & 74.4                    & 83.1                    & 90.0                     \\
		M-PROFS               & 76.5          & 89.0          & 95.2          & \textbf{98.5} & 81.1                    & \textbf{88.1}           & 92.7                    & \textbf{95.8}            & \textbf{64.9}           & \textbf{75.8}           & \textbf{84.2}           & \textbf{90.4}            \\
		M-PROFS-HNCM          & \textbf{76.9} & \textbf{89.5} & \textbf{95.3} & \textbf{98.5} & \textbf{81.3}           & 88.0                    & \textbf{93.0}           & \textbf{95.8}            & 64.1                    & 75.0                    & \textbf{84.2}           & 90.3                     \\ 
		\midrule
		SoftTriple \cite{Qian_2019_ICCV}            & 78.3          & 90.3          & 95.9          & -             & 84.5                    & 90.7                    & 94.5                    & 96.9                     & 65.4                    & 76.4                    & 84.5                    & 90.4                     \\
		SoftTriple-PROFS      & \color[HTML]{3166FF}78.6          & \color[HTML]{3166FF}91.4          & 96.0          & \color[HTML]{3166FF}99.1          & \color[HTML]{3166FF}86.1                    & 91.9                    & 94.7                    & \color[HTML]{3166FF}97.4                     &\color[HTML]{FE0000} \textbf{66.0}           & \color[HTML]{3166FF}\textbf{76.8}           & \color[HTML]{3166FF}\textbf{85.0}           & \color[HTML]{3166FF}\textbf{90.7}            \\
		SoftTriple-PROFS-HNCM & \color[HTML]{FE0000}\textbf{78.7} &\color[HTML]{FE0000} \textbf{91.7} & \color[HTML]{FE0000}\textbf{96.8} &\color[HTML]{FE0000} \textbf{99.2} & \color[HTML]{FE0000}\textbf{86.3}           & \color[HTML]{FE0000}\textbf{92.5}           & \color[HTML]{3166FF}\textbf{95.0}           & \color[HTML]{FE0000}\textbf{97.5}            & \color[HTML]{3166FF}65.7                    & 76.2                    & 84.5                    & 90.6                     \\
		\bottomrule
	\end{tabular}}
\end{table*}

\subsection{Ablation Study}
\label{sec:ablation}
\begin{figure}[!h]

\begin{minipage}[b]{1.0\linewidth}
  \centering
  \centerline{\includegraphics[width=1.0\linewidth,keepaspectratio]{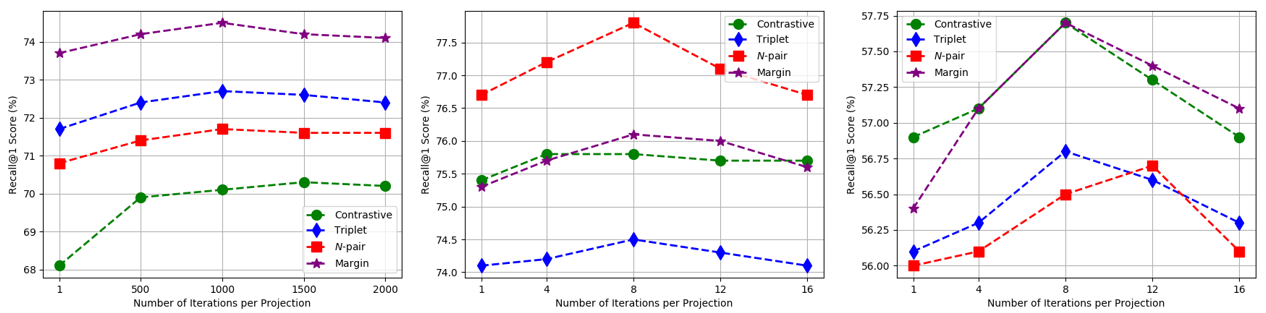}}

\end{minipage}

\caption{The retrieval performance of the PROFS on the Stanford Online Products \cite{oh2016deep} (left), the CARS196 \cite{krause2014submodular} (middle) and the CUB-200-2011 \cite{wah2011caltech} (right) for different number of iterations used to compute projection.}
\label{fig:ablation_study}
\end{figure}
We present to ablation studies on the approximation of the projection operator and the effectiveness of the regularization term. For the computational complexity, the proposed PROFS without HNCM framework does not bring any additional computational cost. On the other hand, HNCM only introduces scan of class representatives during batch construction and has $0.008\%$ (CUB\&CARS) and $0.9\%$ (SOP) increase in overall computation time. Moreover, we have not observed any significant increase or decrease in the convergence rate owing to re-utilization of the class samples.
\subsubsection{Number of Projection Iterations}	
We perform ablation study to determine approximately how many training iterations, $M$, required to assure an acceptable approximation of the projection operation. Due to the non-convex nature of the problem, enforcing convergence for each projection might lead to ill-configured parameters that the proceeding projections cannot recover.  
	
To examine effect of $M$, we apply PROFS framework without hard negative class mining to the contrastive, triplet, $N$-pair and margin loss functions on SOP, Cars196 and CUB-200-2011. We use 2 positive images, $I{=}2$, per class in mini-batches for each loss function to have consistency among the loss terms. The retrieval performance curves for varying $M$ values are plotted Fig. \ref{fig:ablation_study}. The retrieval performance of each loss function improves on all three datasets as $M$ increases up to certain values. Exceeding that certain number of iterations leads to possible over-fitting to the subproblems and the performance drops accordingly. Once the performance curves of the different datasets are compared, a heuristic relation between $M$ and the number of classes in a dataset can be deduced. Hence, instead of fine-tuning $M$ parameter for each problem, we derive $M$ as $M{=}\lceil \nicefrac{\rho }{p(y_i) } \rceil$ where  $\lceil{\cdot}\rceil$ is the ceiling function, $p(y_i)$ denotes the probability of observing a class $i$ in a mini-batch of size $B$, and $\rho$ is how many times a class representation is to be used during the minimization. We write $p(y_i){=}\nicefrac{B}{I\,L}$ for a batch containing $I$ samples from each class it contains among $L$ many classes. Considering the plots in Fig. \ref{fig:ablation_study}, we set $\rho{=}6$ throughout the experiments.

\subsubsection{Effectiveness of the Regularization Term}
We analyze the effect of $\tfrac{\lambda}{2}\Vert \theta^{(curr.)}{-}\theta \Vert_2^2$ term by sweeping the $\lambda$ parameter for training settings with Margin loss in GoogLeNet V1. Theoretically $\lambda$ should be as small as possible (Eqn. \eqref{eq:loss_formulation}), thus large values result in performance drop owing to over-regularization. On the other hand, $\lambda{=}0$ implies no regularization for the projection and leads to over-fitting to the subproblems. According to results in Fig. \ref{fig:ablation_lambda}, the performances are similar for the values within $\left[10^{-2},10^{-4}\right]$. Though we pick $10^{-3}$ throughout the experiments, one can choose to adaptively reduce $\lambda$ as the parameters converge. In this way, the relative significance of the loss function for the constraints can be preserved as the parameters converge. 

\begin{figure}[!h]

  \centering
  \centerline{\includegraphics[width=0.83\linewidth]{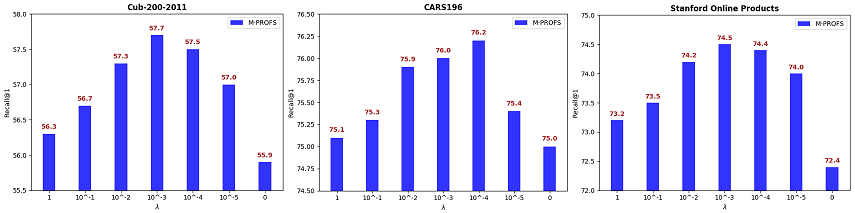}}

%
\caption{Effect of the $\Vert \theta^{(curr.)}{-}\theta \Vert_2^2$ term}
\label{fig:ablation_lambda}
\end{figure}

\subsection{Quantitative Results}
The quantitative results of the proposed PROFS framework trained in GoogLeNet V1 with and without HNCM for the clustering and retrieval tasks on SOP, CARS196 and  CUB-200-2011 datasets are provided in Table \ref{tab:all} together with the baseline methods indicated in \ref{sec:baseline} for comparison. It can be observed that PROFS framework consistently outperforms the associated baseline methods. Compared with the original loss functions, the proposed PROFS framework boosts their performance on each dataset for the clustering and retrieval tasks by up to 3.2\%, 11.7\% and 9.5\% points on NMI, F$_1$ and R@1 metrics, respectively. Additionally, the proposed PROFS framework with the contrastive and triplet loss functions produce competitive results in comparison with superior Margin loss. It supports that considering relaxed feasible sets iteratively is beneficial over solving the entire problem at once even with the basic loss functions. This result is important to support the motivation of our formulation. Furthermore, performance improvements on the loss functions which does not directly fit in our formulation in Eq. \eqref{eq:loss_formulation} show that the batch construction and regularization implication of the proposed formulation can be generalized to the pairwise distance based loss functions.

The performances of the proposed framework integrated to Margin and SoftTriple architectures are provided in Table \ref{tab:all_result_2}. In most settings, the proposed framework improves state-of-the-art. Especially, improving SoftTriple is of great importance, since that performance increase comes from the regularization purely. This result supports the effectiveness of the regularization term to be used to improve generalization. Lastly, utilizing HNCM is more efficient on SOP dataset, since it has almost 100 times larger number of classes than  CARS196 and  CUB-200-2011 datasets.

\section{Conclusion}
We have presented a novel DML formulation based on alternating projections onto the feasible sets which impose relaxed proximity constraints. The resultant framework introduces a simple, yet effective, batch construction scheme and a regularizer to improve the generalization. Notably, the proposed framework is applicable with the pairwise distance based state-of-the-art DML loss functions without introducing any additional computational cost. Extensive evaluations on the benchmark datasets show that the performances of the several state-of-the-art loss functions are improved by the proposed framework. 

	
{\small
\bibliographystyle{ieee_fullname}
\bibliography{library}

\begin{thebibliography}{10}\itemsep=-1pt

\bibitem{abadi2016tensorflow}
Mart{\'\i}n Abadi, Paul Barham, Jianmin Chen, Zhifeng Chen, Andy Davis, Jeffrey
  Dean, Matthieu Devin, Sanjay Ghemawat, Geoffrey Irving, Michael Isard, et~al.
\newblock Tensorflow: a system for large-scale machine learning.
\newblock In {\em OSDI}, volume~16, pages 265--283, 2016.

\bibitem{bauschke2000dykstras}
Heinz~H Bauschke and Adrian~S Lewis.
\newblock Dykstras algorithm with bregman projections: A convergence proof.
\newblock {\em Optimization}, 48(4):409--427, 2000.

\bibitem{bregman1967relaxation}
Lev~M Bregman.
\newblock The relaxation method of finding the common point of convex sets and
  its application to the solution of problems in convex programming.
\newblock {\em USSR computational mathematics and mathematical physics},
  7(3):200--217, 1967.

\bibitem{Cakir_2019_CVPR}
Fatih Cakir, Kun He, Xide Xia, Brian Kulis, and Stan Sclaroff.
\newblock Deep metric learning to rank.
\newblock In {\em The IEEE Conference on Computer Vision and Pattern
  Recognition (CVPR)}, June 2019.

\bibitem{Chen_2019_ICCV}
Guangyi Chen, Tianren Zhang, Jiwen Lu, and Jie Zhou.
\newblock Deep meta metric learning.
\newblock In {\em The IEEE International Conference on Computer Vision (ICCV)},
  October 2019.

\bibitem{chen2017beyond}
Weihua Chen, Xiaotang Chen, Jianguo Zhang, and Kaiqi Huang.
\newblock Beyond triplet loss: a deep quadruplet network for person
  re-identification.
\newblock In {\em Proceedings of the IEEE Conference on Computer Vision and
  Pattern Recognition}, pages 403--412, 2017.

\bibitem{Do_2019_CVPR}
Thanh-Toan Do, Toan Tran, Ian Reid, Vijay Kumar, Tuan Hoang, and Gustavo
  Carneiro.
\newblock A theoretically sound upper bound on the triplet loss for improving
  the efficiency of deep distance metric learning.
\newblock In {\em The IEEE Conference on Computer Vision and Pattern
  Recognition (CVPR)}, June 2019.

\bibitem{duan2017deep}
Yueqi Duan, Jiwen Lu, Jianjiang Feng, and Jie Zhou.
\newblock Deep localized metric learning.
\newblock {\em IEEE Transactions on Circuits and Systems for Video Technology},
  28(10):2644--2656, 2017.

\bibitem{duan2018deep}
Yueqi Duan, Wenzhao Zheng, Xudong Lin, Jiwen Lu, and Jie Zhou.
\newblock Deep adversarial metric learning.
\newblock In {\em Proceedings of the IEEE Conference on Computer Vision and
  Pattern Recognition}, pages 2780--2789, 2018.

\bibitem{ge2018deep}
Weifeng Ge.
\newblock Deep metric learning with hierarchical triplet loss.
\newblock In {\em Proceedings of the European Conference on Computer Vision
  (ECCV)}, pages 269--285, 2018.

\bibitem{hadsell2006dimensionality}
Raia Hadsell, Sumit Chopra, and Yann LeCun.
\newblock Dimensionality reduction by learning an invariant mapping.
\newblock In {\em null}, pages 1735--1742. IEEE, 2006.

\bibitem{harwood2017smart}
Ben Harwood, BG Kumar, Gustavo Carneiro, Ian Reid, Tom Drummond, et~al.
\newblock Smart mining for deep metric learning.
\newblock In {\em Proceedings of the IEEE International Conference on Computer
  Vision}, pages 2821--2829, 2017.

\bibitem{he2016identity}
Kaiming He, Xiangyu Zhang, Shaoqing Ren, and Jian Sun.
\newblock Identity mappings in deep residual networks.
\newblock In {\em European conference on computer vision}, pages 630--645.
  Springer, 2016.

\bibitem{hu2014discriminative}
Junlin Hu, Jiwen Lu, and Yap-Peng Tan.
\newblock Discriminative deep metric learning for face verification in the
  wild.
\newblock In {\em Proceedings of the IEEE conference on computer vision and
  pattern recognition}, pages 1875--1882, 2014.

\bibitem{huang2016local}
Chen Huang, Chen~Change Loy, and Xiaoou Tang.
\newblock Local similarity-aware deep feature embedding.
\newblock In {\em Advances in neural information processing systems}, pages
  1262--1270, 2016.

\bibitem{Jacob_2019_ICCV}
Pierre Jacob, David Picard, Aymeric Histace, and Edouard Klein.
\newblock Metric learning with horde: High-order regularizer for deep
  embeddings.
\newblock In {\em The IEEE International Conference on Computer Vision (ICCV)},
  October 2019.

\bibitem{kim2018attention}
Wonsik Kim, Bhavya Goyal, Kunal Chawla, Jungmin Lee, and Keunjoo Kwon.
\newblock Attention-based ensemble for deep metric learning.
\newblock In {\em Proceedings of the European Conference on Computer Vision
  (ECCV)}, pages 736--751, 2018.

\bibitem{kingma2014adam}
Diederik~P Kingma and Jimmy Ba.
\newblock Adam: A method for stochastic optimization.
\newblock {\em arXiv preprint arXiv:1412.6980}, 2014.

\bibitem{krause2014submodular}
Andreas Krause and Daniel Golovin.
\newblock Submodular function maximization., 2014.

\bibitem{law2013quadruplet}
Marc~T Law, Nicolas Thome, and Matthieu Cord.
\newblock Quadruplet-wise image similarity learning.
\newblock In {\em Proceedings of the IEEE International Conference on Computer
  Vision}, pages 249--256, 2013.

\bibitem{law2017deep}
Marc~T Law, Raquel Urtasun, and Richard~S Zemel.
\newblock Deep spectral clustering learning.
\newblock In {\em Proceedings of the 34th International Conference on Machine
  Learning-Volume 70}, pages 1985--1994. JMLR. org, 2017.

\bibitem{lin2018deep}
Xudong Lin, Yueqi Duan, Qiyuan Dong, Jiwen Lu, and Jie Zhou.
\newblock Deep variational metric learning.
\newblock In {\em Proceedings of the European Conference on Computer Vision
  (ECCV)}, pages 689--704, 2018.

\bibitem{movshovitz2017no}
Yair Movshovitz-Attias, Alexander Toshev, Thomas~K Leung, Sergey Ioffe, and
  Saurabh Singh.
\newblock No fuss distance metric learning using proxies.
\newblock In {\em Proceedings of the IEEE International Conference on Computer
  Vision}, pages 360--368, 2017.

\bibitem{normalization2015accelerating}
Batch Normalization.
\newblock Accelerating deep network training by reducing internal covariate
  shift.
\newblock {\em CoRR.--2015.--Vol. abs/1502.03167.--URL: http://arxiv.
  org/abs/1502.03167}, 2015.

\bibitem{oh2016deep}
Hyun Oh~Song, Yu Xiang, Stefanie Jegelka, and Silvio Savarese.
\newblock Deep metric learning via lifted structured feature embedding.
\newblock In {\em Proceedings of the IEEE Conference on Computer Vision and
  Pattern Recognition}, pages 4004--4012, 2016.

\bibitem{opitz2017bier}
Michael Opitz, Georg Waltner, Horst Possegger, and Horst Bischof.
\newblock Bier-boosting independent embeddings robustly.
\newblock In {\em Proceedings of the IEEE International Conference on Computer
  Vision}, pages 5189--5198, 2017.

\bibitem{pang2015nonconvex}
CH Pang.
\newblock Nonconvex set intersection problems: From projection methods to the
  newton method for super-regular sets.
\newblock {\em arXiv preprint arXiv:1506.08246}, 2015.

\bibitem{paszke2017automatic}
Adam Paszke, Sam Gross, Soumith Chintala, Gregory Chanan, Edward Yang, Zachary
  DeVito, Zeming Lin, Alban Desmaison, Luca Antiga, and Adam Lerer.
\newblock Automatic differentiation in pytorch.
\newblock 2017.

\bibitem{perrot2015regressive}
Micha{\"e}l Perrot and Amaury Habrard.
\newblock Regressive virtual metric learning.
\newblock In {\em Advances in Neural Information Processing Systems}, pages
  1810--1818, 2015.

\bibitem{Qian_2019_ICCV}
Qi Qian, Lei Shang, Baigui Sun, Juhua Hu, Hao Li, and Rong Jin.
\newblock Softtriple loss: Deep metric learning without triplet sampling.
\newblock In {\em The IEEE International Conference on Computer Vision (ICCV)},
  October 2019.

\bibitem{rippel2015metric}
Oren Rippel, Manohar Paluri, Piotr Dollar, and Lubomir Bourdev.
\newblock Metric learning with adaptive density discrimination.
\newblock {\em International Conference on Learning Representations (ICLR)},
  2016.

\bibitem{Roth_2019_ICCV}
Karsten Roth, Biagio Brattoli, and Bjorn Ommer.
\newblock Mic: Mining interclass characteristics for improved metric learning.
\newblock In {\em The IEEE International Conference on Computer Vision (ICCV)},
  October 2019.

\bibitem{russakovsky2015imagenet}
Olga Russakovsky, Jia Deng, Hao Su, Jonathan Krause, Sanjeev Satheesh, Sean Ma,
  Zhiheng Huang, Andrej Karpathy, Aditya Khosla, Michael Bernstein, et~al.
\newblock Imagenet large scale visual recognition challenge.
\newblock {\em International journal of computer vision}, 115(3):211--252,
  2015.

\bibitem{Sanakoyeu_2019_CVPR}
Artsiom Sanakoyeu, Vadim Tschernezki, Uta Buchler, and Bjorn Ommer.
\newblock Divide and conquer the embedding space for metric learning.
\newblock In {\em The IEEE Conference on Computer Vision and Pattern
  Recognition (CVPR)}, June 2019.

\bibitem{schroff2015facenet}
Florian Schroff, Dmitry Kalenichenko, and James Philbin.
\newblock Facenet: A unified embedding for face recognition and clustering.
\newblock In {\em Proceedings of the IEEE conference on computer vision and
  pattern recognition}, pages 815--823, 2015.

\bibitem{sohn2016improved}
Kihyuk Sohn.
\newblock Improved deep metric learning with multi-class n-pair loss objective.
\newblock In {\em Advances in Neural Information Processing Systems}, pages
  1857--1865, 2016.

\bibitem{solomon2015convolutional}
Justin Solomon, Fernando De~Goes, Gabriel Peyr{\'e}, Marco Cuturi, Adrian
  Butscher, Andy Nguyen, Tao Du, and Leonidas Guibas.
\newblock Convolutional wasserstein distances: Efficient optimal transportation
  on geometric domains.
\newblock {\em ACM Transactions on Graphics (TOG)}, 34(4):66, 2015.

\bibitem{song2017deep}
Hyun~Oh Song, Stefanie Jegelka, Vivek Rathod, and Kevin Murphy.
\newblock Deep metric learning via facility location.
\newblock In {\em Computer Vision and Pattern Recognition (CVPR)}, volume~8,
  2017.

\bibitem{Suh_2019_CVPR}
Yumin Suh, Bohyung Han, Wonsik Kim, and Kyoung~Mu Lee.
\newblock Stochastic class-based hard example mining for deep metric learning.
\newblock In {\em The IEEE Conference on Computer Vision and Pattern
  Recognition (CVPR)}, June 2019.

\bibitem{sun2014deep}
Yi Sun, Yuheng Chen, Xiaogang Wang, and Xiaoou Tang.
\newblock Deep learning face representation by joint
  identification-verification.
\newblock In {\em Advances in neural information processing systems}, pages
  1988--1996, 2014.

\bibitem{szegedy2015going}
Christian Szegedy, Wei Liu, Yangqing Jia, Pierre Sermanet, Scott Reed, Dragomir
  Anguelov, Dumitru Erhan, Vincent Vanhoucke, and Andrew Rabinovich.
\newblock Going deeper with convolutions.
\newblock In {\em Proceedings of the IEEE conference on computer vision and
  pattern recognition}, pages 1--9, 2015.

\bibitem{ustinova2016learning}
Evgeniya Ustinova and Victor Lempitsky.
\newblock Learning deep embeddings with histogram loss.
\newblock In {\em Advances in Neural Information Processing Systems}, pages
  4170--4178, 2016.

\bibitem{wah2011caltech}
Catherine Wah, Steve Branson, Peter Welinder, Pietro Perona, and Serge
  Belongie.
\newblock The caltech-ucsd birds-200-2011 dataset.
\newblock 2011.

\bibitem{wang2017deep}
Jian Wang, Feng Zhou, Shilei Wen, Xiao Liu, and Yuanqing Lin.
\newblock Deep metric learning with angular loss.
\newblock In {\em 2017 IEEE International Conference on Computer Vision
  (ICCV)}, pages 2612--2620. IEEE, 2017.

\bibitem{Wang_2019_CVPR_MS}
Xun Wang, Xintong Han, Weilin Huang, Dengke Dong, and Matthew~R. Scott.
\newblock Multi-similarity loss with general pair weighting for deep metric
  learning.
\newblock In {\em The IEEE Conference on Computer Vision and Pattern
  Recognition (CVPR)}, June 2019.

\bibitem{Wang_2019_CVPR}
Xinshao Wang, Yang Hua, Elyor Kodirov, Guosheng Hu, Romain Garnier, and Neil~M.
  Robertson.
\newblock Ranked list loss for deep metric learning.
\newblock In {\em The IEEE Conference on Computer Vision and Pattern
  Recognition (CVPR)}, June 2019.

\bibitem{weinberger2006distance}
Kilian~Q Weinberger, John Blitzer, and Lawrence~K Saul.
\newblock Distance metric learning for large margin nearest neighbor
  classification.
\newblock In {\em Advances in neural information processing systems}, pages
  1473--1480, 2006.

\bibitem{wu2017sampling}
Chao-Yuan Wu, R Manmatha, Alexander~J Smola, and Philipp Krahenbuhl.
\newblock Sampling matters in deep embedding learning.
\newblock In {\em Proceedings of the IEEE International Conference on Computer
  Vision}, pages 2840--2848, 2017.

\bibitem{xing2003distance}
Eric~P Xing, Michael~I Jordan, Stuart~J Russell, and Andrew~Y Ng.
\newblock Distance metric learning with application to clustering with
  side-information.
\newblock In {\em Advances in neural information processing systems}, pages
  521--528, 2003.

\bibitem{xuan2018deep}
Hong Xuan, Richard Souvenir, and Robert Pless.
\newblock Deep randomized ensembles for metric learning.
\newblock In {\em Proceedings of the European Conference on Computer Vision
  (ECCV)}, pages 723--734, 2018.

\bibitem{yi2014deep}
Dong Yi, Zhen Lei, Shengcai Liao, and Stan~Z Li.
\newblock Deep metric learning for person re-identification.
\newblock In {\em 2014 22nd International Conference on Pattern Recognition},
  pages 34--39. IEEE, 2014.

\bibitem{Yu_2019_ICCV}
Baosheng Yu and Dacheng Tao.
\newblock Deep metric learning with tuplet margin loss.
\newblock In {\em The IEEE International Conference on Computer Vision (ICCV)},
  October 2019.

\bibitem{yuan2017hard}
Yuhui Yuan, Kuiyuan Yang, and Chao Zhang.
\newblock Hard-aware deeply cascaded embedding.
\newblock In {\em Proceedings of the IEEE international conference on computer
  vision}, pages 814--823, 2017.

\bibitem{zhang2016embedding}
Xiaofan Zhang, Feng Zhou, Yuanqing Lin, and Shaoting Zhang.
\newblock Embedding label structures for fine-grained feature representation.
\newblock In {\em Proceedings of the IEEE Conference on Computer Vision and
  Pattern Recognition}, pages 1114--1123, 2016.

\bibitem{zhao2018adversarial}
Yiru Zhao, Zhongming Jin, Guo-jun Qi, Hongtao Lu, and Xian-sheng Hua.
\newblock An adversarial approach to hard triplet generation.
\newblock In {\em Proceedings of the European Conference on Computer Vision
  (ECCV)}, pages 501--517, 2018.

\bibitem{Zheng_2019_CVPR}
Wenzhao Zheng, Zhaodong Chen, Jiwen Lu, and Jie Zhou.
\newblock Hardness-aware deep metric learning.
\newblock In {\em The IEEE Conference on Computer Vision and Pattern
  Recognition (CVPR)}, June 2019.

\end{thebibliography}
}

\end{document}